\documentclass[journal]{IEEEtran}%双栏

\usepackage{hyperref}
\usepackage{url}
\usepackage{setspace}
\usepackage{amssymb, amsmath, amsthm, graphicx, enumerate, amsfonts}
\usepackage{algorithm}
\usepackage{algpseudocode}
\usepackage{multirow}
\newtheorem{The}{Theorem}

\begin{document}
\title{Global Hashing System for Fast Image Search}


%\author{\IEEEauthorblockN{Dayong Tian and Dacheng Tao,~\IEEEmembership{Fellow,~IEEE}}
%\IEEEauthorblockA{Faculty of Engineering and Information Technology, University of Technology, Sydney, Broadway, NSW, Australia, 2007}% <-this % stops an unwanted space
%\thanks{Manuscript received December 1, 2012; revised December 27, 2012. Corresponding author: D. Tian (email: dayong_tian@uts.student.edu.au).}}
\author{\IEEEauthorblockN{Dayong Tian and
Dacheng Tao,~\IEEEmembership{Fellow,~IEEE}}
\thanks{The authors are with the Centre for Quantum Computation \& Intelligent Systems and the Faculty of Engineering and Information Technology, University of Technology Sydney, 81 Broadway Street, Ultimo, NSW 2007, Australia (email: dayong.tian@student.uts.edu.au, dacheng.tao@uts.edu.au). Corresponding author: D. Tian (email: dayong.tian@student.uts.edu.au).\copyright 2016 IEEE. Personal use of this material is permitted. Permission from IEEE must be obtained for all other uses, in any current or future media, including reprinting/republishing this material for advertising or promotional purposes, creating new collective works, for resale or redistribution to servers or lists, or reuse of any copyrighted component of this work in other works.}\thanks{\copyright 2017 IEEE. Personal use of this material is permitted. Permission from IEEE must be obtained for all other uses, in any current or future media, including reprinting/republishing this material for advertising or promotional purposes, creating new collective works, for resale or redistribution to servers or lists, or reuse of any copyrighted component of this work in other works.}}

\IEEEtitleabstractindextext{%
\begin{abstract}
Hashing methods have been widely investigated for fast approximate nearest neighbor searching in large datasets. Most existing methods use binary vectors in lower dimensional spaces to represent data points, which are usually real vectors of higher dimensionality. However, according to Shannon's Source Coding Theorem (SSCT) in information theory, it is logical to represent low-dimensional real vectors with high-dimensional binary vectors, since a binary bit contains less information than a real number. We design a novel hashing method based on this principle. Data points are first embedded in a low-dimensional space, and then the Global Positioning System (GPS) method is introduced but modified for hashing. We devise data-independent and data-dependent methods to distribute the ``satellites'' at appropriate locations. Benefitting from the rationale of SSCT and rules on distributing satellites in a GPS, our data-dependent method outperforms other methods in different-sized datasets from 100K to 10M. By incorporating the orthogonality of the code matrix, both our data-independent and data-dependent methods are particularly impressive in experiments on longer bits.
\end{abstract}

}

\maketitle

\IEEEdisplaynontitleabstractindextext
\IEEEpeerreviewmaketitle
%\doublespacing%单栏使用！

\section{Introduction}

Hashing methods are efficient for approximate nearest neighbor (ANN) searching, which is important in computer vision ~\cite{CV1}\cite{CV2}\cite{CV3}\cite{Tiny} and machine learning~\cite{ML1}\cite{ML2}\cite{ML3}\cite{ML4}. Hashing methods map original input data points to binary hash codes while preserving their mutual distances; that is, the binary strings of similar data points in the original feature space should have low Hamming distances. Hashing with short codes can substantially reduce storage requirements and boost the ANN searching speed.\medskip\\
\indent Popular hashing methods can be categorized into two groups according to their dependence on data. The most well-known data-independent hashing methods are Locality-Sensitive Hashing (LSH)~\cite{LSH} and its variances, \emph{e.g.}, those adopting cosine similarity~\cite{cos_sim} and kernel similarity~\cite{KLSH}. The main drawback of these methods is the demand of more bits per hashing table, due to randomized hashing~\cite{IMH}.\medskip\\
\indent Data-dependent methods have become popular in the machine learning community. Spectral Hashing (SH)~\cite{SH}, one of the most popular data-dependent methods, generate hashing codes by solving the relaxed mathematical problem to circumvent the computation of pairwise distances in the whole dataset, i.e, the affinity matrix and the constraints that lead a NP-hard problem. Anchor Graph Hashing (AGH)~\cite{AGH} optimizes the object function of SH by using anchor points to construct a highly sparse affinity matrix. Discrete Graph Hashing (DGH)~\cite{DGH} follows this idea and incorporates the orthogonality of hashing code matrix. There are also methods based on linear projections of Principal Component Analysis (PCA)~\cite{ITQ}\cite{IsoH}\cite{PCAHashing} or Linear Discriminant Analysis~\cite{LDAHashing} and those hashing in kernel space, such as binary reconstructive embeddings (BRE)~\cite{BRE}, random maximum margin hashing (RMMH)~\cite{RMMH} and kernel-based supervised hashing (KSH)~\cite{CV2}. Unlike the ITQ that rotates the projection matrix obtained by PCA to minimize the loss function, the Neighborhood Discriminant Hashing (NDH)~\cite{NDH} incorporate the computation of the projection matrix during the minimization procedure. In general, the linear dimensionality reduction techniques, such as PCA, is inferior to nonlinear manifold learning methods which are able to more effectively preserve the local structure of the input data without assuming global linearity~\cite{IMH_Motivation}. However, the nonlinear manifold techniques may be intractable for large datasets because of their high computation costs. To address this problem, Inductive Manifold Hashing (IMH)~\cite{IMH}\cite{IMH2} learns the nonlinear manifold on a small subset and inductively insert the remainder of data. Besides, hashing methods focus on the image representations have been developed recently. For example, RZhang~\emph{et al.}~\cite{TIP_Guanshui1} unifies the feature extraction and the hashing function learning. Zhang~\emph{et al.}~\cite{TIP_Guanshui2} and Liu~\emph{et al}~\cite{TIP_Guanshui3} develop their methods on multiple representations. \\
\indent However, the main theoretical deficit in the data-dependent methods is that they fail to conform to Shannon's Source Coding Theorem (SSCT)~\cite{IT}. In practice, an image in the dataset is usually represented by a descriptor, \emph{e.g.}, SIFT~\cite{SIFT} or GIST~\cite{GIST} descriptor with more than 128-dimensional 8-bit characters or 32-bit single real numbers in a computer. In information theory~\cite{IT}, entropy is the average amount of information contained in a message, which, in this context, refers to a descriptor vector or binary code vector. According to SSCT, the code length should be no less than the Shannon entropy of original data points. Without ambiguity in this paper, entropy refers to Shannon entropy. The entropy is defined as    $H\left( \Xi  \right) =  - \sum\nolimits_i {P\left( {\Xi  = {\xi _i}} \right){{\log }_2}P\left( {\Xi  = {\xi _i}} \right)} $, where $\Xi$  is a random variable and $P\left( {\Xi  = {\xi _i}} \right)$ is the probability of $\Xi  = {\xi _i}$. For instance, by assuming uniform distribution, the entropy of a 64-dimensional 8-bit character vector is 512, which means 512-bit binary strings are needed.\medskip\\
\indent Exploiting this principle, we first reduce the dimensionality of the original data points, i.e., the descriptor vectors, by PCA. Then, the projections on the first $d$ principle components are encoded by $c$-dimensional binary code, where $c>d$. Hence, we need an over-determined system that can uniquely position every data point. This is similar to Global Positioning Systems (GPS)~\cite{GPS}, which use dozens of satellites to position a receiver on the $2D$ Earth surface. Since our method is directly inspired by GPS, we name it the Global Hashing System (GHS). We tackle the major issue of how to distribute satellites and propose two methods: one data-dependent method and one data-independent method. Unlike most existing methods~\cite{ITQ}\cite{SH}\cite{PCAHashing} that handle the degraded version of orthogonality of code matrix in continuous domain, both our methods approximate the orthogonal code matrix directly in binary domain, which leads better performance on long-bit experiments. Note that although SH can be regarded as assigning more bits to PCA directions along which the data have greater ranges, it is somewhat heuristic~\cite{ITQ}.\medskip\\
\indent After the satellites are well distributed, the distances from data points to each satellite (to simplify following discussion, this distance is denoted as D2S hereafter) are sorted separately. The nearest half is denoted as -1 while the other half is denoted as 1. Hence, our method can generate balanced code matrix easily. Although a balanced code matrix is considered to be one of the two conditions for good codes~\cite{SH}, it is rarely considered because it usually results in a NP-hard problem.

\section{Methodology}

Let us define the used notations. A set of $n$ data points in a $D$-dimensional space is represented by $\left\{ {{{\bf{x}}_1},...,{{\bf{x}}_n}} \right\}$, ${{\bf{x}}_i} \in {\mathbb{R}^D}$ which form the rows of data matrix ${\bf{X}} \in {\mathbb{R}^{n \times D}}$. ${\bf{W}} \in {\mathbb{R}^{D \times d}}$ is obtained by the first $d$ eigenvectors of the data covariance matrix ${{\bf{X}}^\top}{\bf{X}}$. ${\bf{Y}} = {\bf{XW}}$ and ${\bf{y}}_i$ is the $i$th row vector of $\bf{Y}$. A binary code corresponding to ${\bf{x}}_i$ is defined by ${{\bf{b}}_i} = {\left\{ { - 1, + 1} \right\}^c}$, where $c$ is the length of the code and the code matrix ${\bf{B}} = {\left[ {{\bf{b}}_1^\top,...,{\bf{b}}_c^\top} \right]^\top}$.

\subsection{Global Positioning/Coding System}
A satellite in a GPS has the ability to measure the distance between itself and a signal receiver on Earth surface. This results in a circle on which every point has the same distance to this satellite as the receiver. Hence, at least three satellites are needed to determine the true position which is the unique intersection of three such circles. More generally, a $d$-dimensional point can be determined by its Euclidean distances to $d+1$ other points in this space~\cite{GPS1}.\medskip\\
\indent In our GHS, each satellite only has 1-bit to record the Euclidean distances. That is, the receivers far from a satellite are denoted as 1 while the nearby ones are denoted as -1. Hence, our hashing function can be defined as:
\begin{equation}\label{eq:hf}
h\left( {{{\bf{y}}_i} - {{\bf{s}}_j}} \right) = \left\{ \begin{array}{l}
 - 1, \qquad {\left\| {{{\bf{y}}_i} - {{\bf{s}}_j}} \right\| \le f\left( {{{\left\| {{\bf{Y}} - {{\bf{1}}^{n \times 1}}{{\bf{s}}_j}} \right\|}_c}} \right)}\\
+1, \qquad {\left\| {{{\bf{y}}_i} - {{\bf{s}}_j}} \right\| > f\left( {{{\left\| {{\bf{Y}} - {{\bf{1}}^{n \times 1}}{{\bf{s}}_j}} \right\|}_c}} \right)}
\end{array} \right. ,
\end{equation}
where ${\left\| {\bf{A}} \right\|_c}$ computes the Frobenius norm of each row of $\bf{A}$ and $f$ can be any proper functions that return a positive real number. Here $median(·)$ is adopted to generate a balanced code matrix. ${\bf{s}}_j$ is the coordinate of the $j$th satellite and it forms up the $j$th row of satellite matrix $\bf{S}$.
\subsection{Data-dependent method (GHS-DD)}
Formally, our hashing model can be described as:
\begin{equation}\label{eq:model}
\begin{array}{l}
\mathop {\arg \min }\limits_{\left\{ {{{\bf{s}}_j}} \right\}} \sum\limits_{i = 1}^{n - 1} {\sum\limits_{i' = i + 1}^n {{e^{ - {{\left\| {{{\bf{y}}_i} - {{\bf{y}}_{i'}}} \right\|}^2}}}} } \\
 \qquad\qquad\qquad\left( {\sum\limits_{p = 1}^c {\left\| {h\left( {{{\bf{y}}_i} - {{\bf{s}}_i}} \right) - h\left( {{{\bf{y}}_{i'}} - {{\bf{s}}_j}} \right)} \right\|} } \right).
\end{array}
\end{equation}
Randomly setting ${\bf{s}}_j$ does not produce satisfactory results. Furthermore, Eq.~\eqref{eq:model} requires pairwise distance between each pair of data points, which leads heavy burden in storage and computation. Inspired by ITQ, we circumvent it by minimizing the quantization loss.\\
At first, let us consider following quantization loss:
\begin{equation}\label{eq:sim1}
\mathop {\arg \min }\limits_{{B_{ij}} \in \left\{ { - 1,1} \right\},{{\bf{s}}_j}} \sum\limits_{i = 1}^n {\sum\limits_{j = 1}^c {{{\left( {\frac{{{B_{ij}} + 1}}{2} - \left\| {{{\bf{y}}_i} - {{\bf{s}}_j}} \right\|} \right)}^2}} }.
\end{equation}
Because $\left\| {{{\bf{y}}_i} - {{\bf{s}}_j}} \right\|$ is always non-negative, we scale and shift B to $[0,1]$. The underlying reasonability of Eq.~\eqref{eq:sim1} is similar to ITQ. To uniquely position a data point in $d$-dimensional space, at least $d+1$ satellites are required and the locations of these satellites should satisfy following condition~\cite{GPS1}:
\begin{equation}\label{eq:cond}
{\rm{rank}}\left( {\left[ {\begin{array}{*{20}{c}}
\Gamma \quad \theta
\end{array}} \right]} \right){\rm{ = }}d,
\end{equation}
where $\Gamma  = \left[ {{{\bf{s}}_2};...;{{\bf{s}}_{d+1}}} \right]$ and $\theta  = \left[ {{{\bf{s}}_2} - {{\bf{s}}_1};...;{{\bf{s}}_{d+1}} - {{\bf{s}}_1}} \right]$. Eq.~\eqref{eq:cond} is called the existence and uniqueness condition for GPS solution~\cite{GPS1}. It can be satisfied by initializing an orthogonal $\Gamma$. Hence, we create $g$ groups of satellites. Within each group, there are $d+1$ satellites, $d$ of which are orthogonal to each other.  We define $\rho : = c/\left( {d + 1} \right)$, a parameter discussed in Section~\ref{sec:parameter}. Note that no more than $d$ mutual orthogonal vectors in a $d$-dimensional space. Each group is rotated by an orthogonal matrix ${\bf{R}}_k$ to find the best location, which gives the following model:
\begin{equation}\label{eq:GHS-DD}
\begin{array}{l}
\mathop {\arg \min }\limits_{\scriptstyle{B_{ij}} \in \left\{ { - 1,1} \right\}\hfill\atop
\scriptstyle{\beta _j},{\alpha _j},{{\bf{R}}_k}\hfill} E = \sum\limits_{i = 1}^n {\sum\limits_{j = 1}^c {\sum\limits_{k = 1}^g {{\delta _k}\left( {{{\bf{s}}_j}} \right){{\left( {{B_{ij}} + {\beta _j} - {\alpha _j}\left\| {{{\bf{y}}_i} - {{\bf{s}}_j}{{\bf{R}}_k}} \right\|} \right)}^2}} } } \;\;\;\;\;\\
 \qquad\qquad\qquad\qquad s.t.{\kern 1pt}  \qquad {\bf{1B}} = {\bf{0}},{\kern 1pt} {\kern 1pt} {{\bf{R}}_k}^\top{{\bf{R}}_k} = {\bf{I}},
\end{array}
\end{equation}
where ${\delta_k}$ is an indicator function. ${\delta _k}\left( {{{\bf{s}}_j}} \right) = 1$, if ${{\bf{s}}_j} \in {\rm{Group}}{\kern 1pt} {\kern 1pt} {\kern 1pt} k$ and ${\delta _k}\left( {{{\bf{s}}_j}} \right) = 0$, if ${{\bf{s}}_j} \notin {\rm{Group}}{\kern 1pt} {\kern 1pt} {\kern 1pt} k$. $\alpha_j$ and $\beta_j$ are used to transform the values of D2S into a proper interval. Eq.~\eqref{eq:GHS-DD} is minimized by iterative minimization.\medskip\\
\textbf{Initialization}. In each group, $\Gamma$ is initialized by the left singular vectors of a $d\times d$ random matrix, so does ${\bf{R}}_k$. Another random $1\times d$ vector is added into each group.\medskip\\
\textbf{Update} $B_{ij}$. The $j$th column of $\bf{B}$ is calculated by Eq.~\eqref{eq:hf}.\medskip\\
\textbf{Update} $\alpha_j$. Take the partial derivative with respect to $\alpha_j$, resulting
\begin{equation}
{\alpha _j} = \frac{{\sum\limits_{i = 1}^n {\sum\limits_{k = 1}^g {{\delta _k}\left( {{{\bf{s}}_j}} \right)\left( {{B_{ij}} + {\beta _j}} \right)\left\| {{{\bf{y}}_i} - {{\bf{s}}_j}{{\bf{R}}_k}} \right\|} } }}{{\sum\limits_{i = 1}^n {\sum\limits_{k = 1}^g {{\delta _k}\left( {{{\bf{s}}_j}} \right){{\left\| {{{\bf{y}}_i} - {{\bf{s}}_j}{{\bf{R}}_k}} \right\|}^2}} } }}.
\end{equation}
\textbf{Update} $\beta_j$. Similar to $\alpha_j$,
\begin{equation}\label{eq:updatebeta}
{\beta _j} = \frac{1}{n}\sum\limits_{i = 1}^n {\sum\limits_{k = 1}^g {{\delta _k}\left( {{{\bf{s}}_j}} \right)\left( {{\alpha _j}\left\| {{{\bf{y}}_i} - {{\bf{s}}_j}{{\bf{R}}_k}} \right\| - {B_{ij}}} \right)} }.
\end{equation}
Please note when we deduce Eq.~\eqref{eq:updatebeta}, $\sum\nolimits_{k = 1}^g {{\delta _k}\left( {{{\bf{s}}_j}} \right)}  = 1$ is applied.\\
\textbf{Update} ${\bf{R}}_k$. We divide this step to two sub-problems. First, ${\bf{s}}_j{\bf{R}}_k$ is substituted by ${\bf{s}}'_j$ to form up following minimization problem:
\begin{equation}
\mathop {\arg \min }\limits_{{\bf{s}}'_j} \sum\limits_{i = 1}^n {\sum\limits_{j = 1}^c {{{\left( {{B_{ij}} + {\beta _j} - {\alpha _j}\left\| {{{\bf{y}}_i} - {\bf{s}}'_j} \right\|} \right)}^2}} },
\end{equation}
which is equivalent to
\begin{equation}\label{eq:updateR1}
\mathop {\arg \min }\limits_{{\bf{s}}'_j} \sum\limits_{i = 1}^n {\sum\limits_{j = 1}^c {{{\left( {B'_{ij} - \left\| {{{\bf{y}}_i} - {\bf{s}}'_j} \right\|} \right)}^2}} }.
\end{equation}
where $B'_{ij} = \left( {{B_{ij}} + {\beta _j}} \right)/{\alpha _j}$. If we treat ${\bf{s}}'_j$ as a receiver, ${{\bf{y}}_i}$ as satellites and $B'_{ij}$ as the D2S, the solution of Eq.~\eqref{eq:updateR1} is the standard solution of GPS~\cite{GPS2}.\medskip\\
\indent We construct following two matrices for each ${\bf{s}}'_j$: $\overline {\bf{Y}}  = \left[ {{\bf{Y}},{\bf{B}}'_{ \cdot j}} \right]$ and ${\bf{Z}} = diag\left( {\overline {\bf{Y}} {{\overline {\bf{Y}} }^\top}} \right)$, where ${\bf{B}}'_{ \cdot j}$ represents the $j$th column of ${{\bf{B}}'}$ and $diag(\bf{A})$ returns a row vector which contains the diagonal elements of $\bf{A}$. Let ${\overline {\bf{Y}} ^ + } = {\left( {{{\overline {\bf{Y}} }^\top}\overline {\bf{Y}} } \right)^{ - 1}}{\overline {\bf{Y}} ^\top}$. Then solve following quadratic equation about $\Lambda$:
\begin{equation}\label{eq:quadratic}
\begin{aligned}
{\Lambda ^2}{\left( {{{\overline {\bf{Y}} }^ + }{\bf{1}}} \right)^\top}\left( {{{\overline {\bf{Y}} }^ + }{\bf{1}}} \right) +& 2\Lambda \left( {{{\left( {{{\overline {\bf{Y}} }^ + }{{\bf{Z}}^\top}} \right)}^\top}\left( {{{\overline {\bf{Y}} }^ + }{\bf{1}}} \right) - 1} \right)\\
 +& {\left( {{{\overline {\bf{Y}} }^ + }{{\bf{Z}}^\top}} \right)^\top}\left( {{{\overline {\bf{Y}} }^ + }{{\bf{Z}}^\top}} \right) = 0.
\end{aligned}
\end{equation}
Eq.~\eqref{eq:quadratic} usually have two solutions $\Lambda_1$ and $\Lambda_2$, therefore two possible $\overline {{\bf{s}}'_j}$ can be found by $\overline {{\bf{s}}'_j}  = {\overline {\bf{Y}} ^ + }\left( {{{\bf{Z}}^\top} + \Lambda {\bf{1}}} \right)$, where $\overline {{\bf{s}}'_j}  = \left[ {{\bf{s}}'_j,\tau } \right]$ and $\tau$ which is useless in our model is related to D2S. To automatically choose a suitable $\overline {{\bf{s}}'_j}$  from two solutions, we initialize ${{\bf{s}}_j}$ with $\left\| {{{\bf{s}}_j}} \right\| = {r_s}$, where $r_s$ is a positive real constant. The $\overline {{\bf{s}}'_j}$ whose norm is closer to $r_s$ is chosen for following steps. $r_s$ is also used in our data-independent satellite distribution algorithm and discussed in Section~\ref{sec:parameter} along with parameter $\rho$.\medskip\\
\indent After ${\bf{s}}'_j$s are calculated, ${\bf{R}}_k$ is found by minimizing following problem:
\begin{equation}\label{eq:updateR2}
\mathop {\arg \min }\limits_{{{\bf{R}}_k}} \sum\limits_{j = 1}^c {\delta \left( {{{\bf{s}}_j}} \right)} \left\| {{\bf{s}}'_j - {{\bf{s}}_j}{{\bf{R}}_k}} \right\|.
\end{equation}
Eq.~\eqref{eq:updateR2} can be solved by singular value decomposition (SVD). Given ${\bf{S}}'_k$ and ${{\bf{S}}_k}$ which contain ${\bf{s}}'_j$ and ${\bf{s}}_j$ of Group $k$, respectively, through SVD, we can get ${{\bf{L}}_1}{\bf{VL}}_2^\top = {{{\bf{S}}'_k}^\top}{{\bf{S}}_k}$ and ${{\bf{R}}_k} = {{\bf{L}}_2}{\bf{L}}_1^\top.$\medskip\\
\textbf{Convergence}. When $\left| {{E^{k - 1}} - {E^k}} \right| < \varepsilon$  or maximum iteration is reached, the algorithm is terminated, where $\varepsilon$ is a small positive real constant.\medskip\\
\textbf{Output}. $\bf{S}$ and thresholds, \emph{i.e.}, $g\left( {{{\left\| {{\bf{Y}} - {{\bf{1}}^{n \times 1}}{{\bf{s}}_j}} \right\|}_c}} \right)$ in Eq.~\eqref{eq:hf}.\medskip\\
\textbf{Out-of-Sample Hashing}. A new query is projected by $\bf{W}$ and then its distance to each satellite ${\bf{s}}_j$ is cut off by $g\left( {{{\left\| {{\bf{Y}} - {{\bf{1}}^{n \times 1}}{{\bf{s}}_j}} \right\|}_c}} \right)$.
{\setlength{\parskip}{1\baselineskip}
\begin{figure*}[htb]
\begin{center}
%\framebox[4.0in]{$\;$}
%\fbox{\rule[-.5cm]{0cm}{4cm} \rule[-.5cm]{4cm}{0cm}}
\includegraphics[width=1\linewidth]{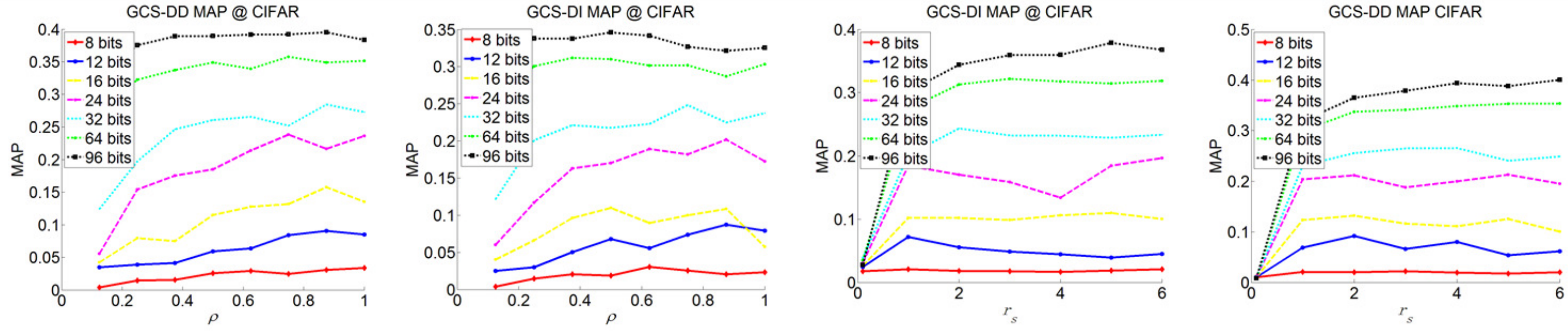}
\end{center}
\caption{MAP on CIFAR-10 dataset for GHS-DI and GHS-DD. When $r_s$ approximates 0, both methods fail to get satisfactory results. The performance of both methods become stable after $r_s$ is larger than 1. On the other hand, GHS-DI gets its best results when $\rho$ is in interval $[0.5,1]$, while it is $[0.7,1]$ for GHS-DD. For $c<16$, the best results appear when $\rho$ approximates 1, because enough amounts of principal components should be selected.}
\label{fig:parameter}
\end{figure*}
}
{%\setlength{\parskip}{1\baselineskip}
\begin{table*}[htb]
\caption{MAP @ CIFAR-10 for parameter setting $c=d+1$ and $c=d$}
\label{tb:rho}
\begin{center}
\begin{tabular}{|c|c|c|c|c|c|c|c|c|}
\hline
	& &	8	&12	&16	&24	&32	&64	&96\\
\hline
\multirow{3}{*}{GHS-DD}&	$c=d$	&0.1890	&0.2232	 &0.2392	&0.2761	&0.3053	 &0.3816	 &0.4131\\
\cline{2-9}
	                       &$c=d+1$	&0.1884	&0.2214	 &0.2412	&0.2806	&0.3089	 &0.3972	 &0.4324\\
\cline{2-9}
		                    &        &-0.32\% &-0.81\%	&0.83\%	&1.60\%	&1.17\%	 &3.93\%	 &4.46\%\\
\hline
\multirow{3}{*}{GHS-DI}	&$c=d$	&0.1543	&0.1838	 &0.2079	&0.2581	&0.2757	&0.3474	 &0.4018\\
\cline{2-9}
	&$c=d+1$	&0.1537	&0.1861	&0.2098	&0.2688	 &0.3008	&0.3653	&0.4144\\
\cline{2-9}
	&	&-0.39\%	&1.24\%	&0.91\%	&3.98\%	&8.34\%	 &4.90\%	&3.04\%\\
\hline

\end{tabular}
\end{center}
\end{table*}
}
\subsection{Data-independent method (GHS-DI)}
Another condition for good code is uncorrelation [23], \emph{i.e.}, ${{\bf{B}}^\top}{\bf{B}} = n{\bf{I}}$. A direct way to satisfy this condition is distributing the satellites such that only one is close to each receiver; that is, there is no intersection among all $\left({\bf s}_j,r_j\right)$ spheres, where $r_j$ is the minimum radius that include the nearby data points of ${\bf s}_j$.  However, in this situation, each receiver only has 1-bit 1. The hamming distance between any pair of receivers is 0 or 2, which means the distance between two data points in input space is not well preserved. What's more, if we strictly satisfy the balance condition as well as uncorrelation condition in this way, at most 2 satellites can be used.\medskip\\
\indent An alternative way is minimizing the intersections of $\left( {{{\bf{s}}_j},{r_j}} \right)$ sphere and $\left( {{{\bf{s}}_{j'}},{r_{j'}}} \right)$ sphere for any $j\neq {j'}$. That is, we put a tolerance for the values of non-diagonal elements of ${{\bf{B}}^\top}{\bf{B}}$. They are allowed to be non-zero numbers with small absolute values.\medskip\\
\indent The intersection of two $d$-dimensional sphere is too difficult to compute, therefore the pairwise distance between each pair of satellites is maximized. Without constraints, the resulting $\left\| {{{\bf{s}}_j}} \right\|$ may be  $+ \infty$. A reasonable constraint is distributing all satellites on the surface of $\left( {{\bf{0}},{r_s}} \right)$ sphere. As there is no prior knowledge about the data, we assume data points are uniformly distributed in a $\left( {{\bf{0}},r} \right)$ sphere. By $\left\| {{{\bf{s}}_1}} \right\| = ... = \left\| {{{\bf{s}}_c}} \right\| = {r_s}$, the D2S of each satellite will be comparable.\medskip\\
\indent Under the abovementioned assumption, minimizing intersections can be achieved by maximizing the pairwise distance between each pair of satellites:
\begin{equation}\label{eq:GHS-DI}
\mathop {\arg \max }\limits_{\left\{ {{{\bf{s}}_j}} \right\}} E: = \sum\limits_{j = 1}^{c - 1} {\sum\limits_{{j'} = j + 1}^c {{{\left\| {{{\bf{s}}_j} - {{\bf{s}}_{j'}}} \right\|}^2}} } \;\;\;\;s.t. \;\;\;\; {\left\| {{{\bf{s}}_j}} \right\|^2} = r_s^2,\forall j.
\end{equation}
Eq.~\eqref{eq:GHS-DI} can be maximized by Gradient Projection Algorithm (GPA)~\cite{GPA}. The GPA iteratively updates ${\bf{s}}_j$ by moving ${\bf{s}}_j$ along the gradient direction of $E$ and projects ${\bf{s}}_j$ to the boundary defined by the constraint (Algorithm~\ref{alg:1}). The gradient of $E$ with respect to ${\bf{s}}_j$ is
\begin{equation}
\frac{\partial E}{\partial {{\bf{s}}_j}} = \left( {c - j} \right){{\bf{s}}_j} - \sum\limits_{{j'} = j + 1}^c {{{\bf{s}}_{j'}}}.
\end{equation}
{%\renewcommand\baselinestretch{1.5}\selectfont
\begin{algorithm}[b]
\caption{~~Data-Independent Satellite Distribution Algorithm}
\label{alg:1}
\begin{algorithmic}[1]
\Require ${\bf{S}} \in {R^{c \times d}}$
\While {$E$ not converged}
\State ${\bf{s}}_j^{t + 1/2} = {\bf{s}}_j^t + \Delta t\partial E/\partial {\bf{s}}_j^t$
\State ${\bf{s}}_j^{t + 1} = {r_s}{\bf{s}}_j^{t + 1/2}/\left\| {{\bf{s}}_j^{t + 1/2}} \right\|$
\EndWhile
\Ensure $\bf{S}$
\end{algorithmic}
\end{algorithm}}
\indent The projection step can be directly implemented by normalizing each ${\bf{s}}_j$. As the orthogonality of $\bf{B}$ is considered, our GHS-DI method usually produces the second best results on experiments of longer hash bits. Actually the way that GHS-DD satisfies Eq.~\eqref{eq:cond} intrinsically incorporates orthogonality. When ${r_s} \to  + \infty$, the hyper-sphere surface that separates the near and far data points can be treated as a hyper-plane. In this situation, with orthogonal $\left\{ {{{\bf{s}}_j}} \right\}$ and assumption of uniform distribution of data points, this property is easy to understand in $2D$ and $3D$ cases. More generally, we have following theorem.
\begin{The}\label{the}
If (1) data points ${{\bf{y}}_i} \in {\mathbb{R}^d}$ are uniformly distributed in a $({\bf{0}},r)$ sphere, (2) ${{\bf{s}}_j} \bot {{\bf{s}}_{j'}}$ and (3) ${r_s} \to  + \infty$, then ${\bf{h}}_j^\top{{\bf{h}}_{j'}} = 0 (j\neq {j'})$, where ${{\bf{h}}_j}$ and ${{\bf{h}}_{j'}}$ are column vectors whose elements are the binary hash codes generated by Eq.~\eqref{eq:hf}.
\end{The}
\begin{proof}
Since the data points are uniformly distributed in a $\left({\bf 0},r\right)$ sphere, without losing generality, let us set ${{\bf{s}}_j} = {r_s}{\left( {1,0,0,...,0} \right)^d}$ and ${{\bf{s}}_{j'}} = {r_s}{\left( {0,1,0,...,0} \right)^d}$. In Eq.~(1), if $\left\| {{{\bf{y}}_i} - {{\bf{s}}_j}} \right\| > {r_s}$, the $i$th element of ${\bf h}_j$ will be set to $1$, otherwise it will be set to $-1$. For any two points ${{\bf{y}}_i}$ and ${{\bf{y}}_j}$ that satisfy $\left\| {{{\bf{y}}_i} - {{\bf{s}}_j}} \right\| = \left\| {{{\bf{y}}_j} - {{\bf{s}}_j}} \right\| = {r_s}$, we have $\left( {{{\bf{y}}_i} - {{\bf{s}}_j}} \right){\left( {{{\bf{y}}_j} - {{\bf{s}}_j}} \right)^\top}/r_s^2 = 1$, when ${r_s} \to  + \infty $. That is, $\cos \theta  \to 1$ which implies $\theta  \to 0$, where $\theta$ is the angle between two unit vectors along ${{\bf{y}}_i} - {{\bf{s}}_j}$ and ${{\bf{y}}_j} - {{\bf{s}}_j}$, respectively. Hence, ${{\bf{y}}_i}$ and ${{\bf{y}}_j}$ locate on a plane $\mathcal{P}$ whose distance to ${\bf s}_j$ is $r_s$.\\
To generate a balanced ${\bf h}_j$, $\mathcal{P}$ should cross the origin and perpendicular to ${\bf s}_j$. Since ${{\bf{s}}_j} \bot {{\bf{s}}_{j'}}$, $\mathcal{P}$ is also perpendicular to $\mathcal{Q}$ which corresponds to ${\bf s}_{j'}$. It is evident that $\mathcal{P}$ and $\mathcal{Q}$ separate the $\left({\bf 0},r\right)$ sphere into four parts with equal volume:
\begin{equation}
\left\{ \begin{array}{l}
\left\{ {{{\bf{y}}_i}|\left\| {{{\bf{y}}_i} - {{\bf{s}}_j}} \right\| > {r_s}} \right\} \cap \left\{ {{{\bf{y}}_i}|\left\| {{{\bf{y}}_i} - {{\bf{s}}_{j'}}} \right\| > {r_s}} \right\}\\ \qquad \qquad \qquad {{\bf{h}}_j}\left( i \right) = 1,{{\bf{h}}_{j'}}\left( i \right) = 1\\
\left\{ {{{\bf{y}}_i}|\left\| {{{\bf{y}}_i} - {{\bf{s}}_j}} \right\| > {r_s}} \right\} \cap \left\{ {{{\bf{y}}_i}|\left\| {{{\bf{y}}_i} - {{\bf{s}}_{j'}}} \right\| < {r_s}} \right\}\\ \qquad \qquad \qquad {{\bf{h}}_j}\left( i \right) = 1,{{\bf{h}}_{j'}}\left( i \right) =  - 1\\
\left\{ {{{\bf{y}}_i}|\left\| {{{\bf{y}}_i} - {{\bf{s}}_j}} \right\| < {r_s}} \right\} \cap \left\{ {{{\bf{y}}_i}|\left\| {{{\bf{y}}_i} - {{\bf{s}}_{j'}}} \right\| > {r_s}} \right\}\\ \qquad \qquad \qquad {{\bf{h}}_j}\left( i \right) =  - 1,{{\bf{h}}_{j'}}\left( i \right) = 1\\
\left\{ {{{\bf{y}}_i}|\left\| {{{\bf{y}}_i} - {{\bf{s}}_j}} \right\| < {r_s}} \right\} \cap \left\{ {{{\bf{y}}_i}|\left\| {{{\bf{y}}_i} - {{\bf{s}}_{j'}}} \right\| < {r_s}} \right\}\\ \qquad \qquad \qquad {{\bf{h}}_j}\left( i \right) =  - 1,{{\bf{h}}_{j'}}\left( i \right) =  - 1
\end{array} \right. .
\end{equation}
Since there are equal number of data points in these four parts, it is easy to verify that ${\bf{h}}_j^\top{{\bf{h}}_{j'}} = 0$.
\end{proof}
\indent In \textbf{Theorem}~\ref{the}, condition (1) and (2) are impractical and therefore only the second sufficient condition can be satisfied by setting $c=d$; however, this contravenes the perspective of SSCT and the existence and uniqueness condition for GPS solution. In Section~\ref{sec:parameter}, we will show $c=d$ usually cannot generate the best results. Although our methods cannot exactly fulfill these three conditions, its superiority of considering the orthogonality was proven by its high F-measure in experiments on longer bits (Section~\ref{sec:exp}).

\subsection{Parameters $r_s$ and $\rho$}
\label{sec:parameter}
\indent There are two key parameters in our methods - $r_s$ and $\rho$. $r_s$ should not be too small. Consider an extreme example that $r_s=0$, then all bits of the points close to the origin will equal to 0 and bits of other points will equal to 1. Obviously, such codes are inefficient.\\
\indent $\rho$ should be moderate. If $\rho$ is too large, the binary codes will gradually lose their ability to encode the values of projections which are real numbers. On the other hand, when $\rho$ becomes small, fewer projections can be used, so the data points reconstructed by these projections cannot approximate the original ones accurately enough.\medskip\\
\indent The mean average precision (MAP) on CIFAR-10 dataset ~\cite{CIFAR-10} with varying $r_s$ and $\rho$ is shown in Fig.~\ref{fig:parameter}. CIFAR-10 comprises of 60K images from the 80 Million Tiny Image dataset~\cite{Tiny} and we use 1024-dimensional GIST descriptor to represent each image. Their PCA projections are normalized by the largest Euclidean norm of all projected data. When testing on different $\rho$s , at most one group containing less than $d+1$ satellites may exist. Based on the results in Fig.~\ref{fig:parameter}, we empirically set $r_s$ as 2 for all experiments and set $\rho$ as 1 for experiments whose $c \leq 16$ , while 0.5 for others.\medskip\\
\indent We also tested our two methods by setting $c=d$ (Table~\ref{tb:rho}). The percentages shown in Table~\ref{tb:rho} denote the improvement by setting $c=d+1$. Referring to Table~\ref{tb:rho}, we observe that for $c>16$, both methods perform $1\%-8\%$ better with $c=d+1$, suggesting that the existence and uniqueness condition for GPS solution is important. For experiment on $c\leq 16$, the situation is opposite, because the number of PCA projections are too small and its effect dominates results. However, the differences are slight in these cases (less than 1\%), so we did not use parameter setting $c=d$ in experiments of Section 4.

\section{Relations to Existing Methods}
During past several years, many state-of-the-art data-dependent hashing methods have been proposed. These methods derive from various motivations. In this section, only those related to our proposed methods are briefly reviewed.
\subsection{Iterative Quantization (ITQ)}
Gong~\emph{et al.}~\cite{ITQ} formulated ITQ as a minimization problem:
\begin{equation}\label{eq:ITQ}
\mathop {\arg \min }\limits_{{\bf{B}},{\bf{R}}} \left\| {{\bf{B}} - {\bf{XWR}}} \right\|_F^2.
\end{equation}
Eq.~\eqref{eq:ITQ} is minimized by iteratively updating $\mathbf{B}$ and $\mathbf{R}$. $\mathbf{R}$ is required to be orthogonal, which can be considered as a rotation to $\mathbf{W}$. IsoH~\cite{IsoH} is directly derived from ITQ by finding a projection with equal variances for different dimensions. HH~\cite{HH} rotates $\bf{W}$; however, unlike ITQ, it uses an auxiliary variable for the code matrix during the iterative optimization and puts an orthogonal constraint on it. Then, the auxiliary variable is thresholded to generate code matrix. ok-means~\cite{okmeans} rotates and scales $\bf{B}$ to minimize the quantization loss. Our method rotates ${\bf S}$ and scales the D2S. ITQ, IsoH and HH use principle components whose number is exactly equal to the bit length of hash codes. That is, they cannot be used to produce hash code that is longer than the data dimension. Theoretically, our methods can produce arbitrary length of hash codes.\medskip\\
\subsection{Inductive Hashing on Manifolds (IMH)}
IMH~\cite{IMH} first generates the Base matrix $\bf{C}$ by K-means clustering. Each column $\bf{C}$ corresponds to a cluster center. Then it embeds $\bf{B}$ into low-dimensional space by manifold learning methods~\cite{manifold1}\cite{manifold2}. The embedding methods affect the performance of IMH. Throughout this paper, t-SNE~\cite{manifold1} is used because it achieved the best results in the authors' experiments~\cite{IMH}. Finally, the embedding for the training data is calculated by
\begin{equation}
{\bf{Y}} = {\overline {\bf{W}} _{{\bf{XB}}}}{{\bf{Y}}_{\bf{B}}},
\end{equation}
where the elements ${\overline {\bf{W}} _{ij}}$ in ${\overline {\bf{W}} _{{\bf{XB}}}}$ is defined as
\begin{equation}\label{eq:IMH}
{\overline {\bf{W}} _{ij}} = \frac{{\exp \left( { - {{\left\| {{{\bf{x}}_i} - {{\bf{c}}_j}} \right\|}^2}/{\sigma ^2}} \right)}}{{\sum\limits_{i = 1}^m {\exp \left( { - {{\left\| {{{\bf{x}}_i} - {{\bf{c}}_j}} \right\|}^2}/{\sigma ^2}} \right)} }}.
\end{equation}
where $\bf{c}_j$ is the $j$th column of $\bf{C}$. Eq.~\eqref{eq:IMH} is quite similar to membership in fuzzy c-means clustering~\cite{FCM}. The embedding for the training data is linear combination of embedding for $\bf{C}$. In our method, each satellite encodes 1-bit according to the distances from itself to the data points and we don't encode the satellites.\medskip\\
\subsection{Spectral Hashing (SH)}
Weiss~\emph{et al.}~\cite{SH} formulated the SH as:
\begin{equation}\label{eq:ITQ}
\begin{array}{l}
\mathop {\arg \min }\limits_{\bf{Y}} \sum\limits_{{{\bf{x}}_i},{{\bf{x}}_j} \in {\bf{X}}} {{e^{ - {{\left\| {{{\bf{x}}_i} - {{\bf{x}}_j}} \right\|}^2}/{\sigma ^2}}}{{\left\| {{{\bf{b}}_i} - {{\bf{b}}_j}} \right\|}^2}} \\
s.t.\quad{\bf{B}} \in {\left\{ { - 1,1} \right\}^{n \times c}},\quad{{\bf{B}}^\top}{\bf{B}} = n{\bf{I}},\quad{{\bf{B}}^\top}{\bf{1}} = 0.
\end{array}
\end{equation}
Eq.~\eqref{eq:model} is similar to Eq.~\eqref{eq:ITQ}. The graph affinity matrix ${\bf{W}}$ with ${{\bf{W}}_{ij}} = \exp \left( { - {{\left\| {{{\bf{x}}_i} - {{\bf{x}}_j}} \right\|}^2}/{\sigma ^2}} \right)$ is intractable for large datasets. SH evaluates $c$ smallest eigenvalues for each PCA direction to create a list of $cD$ eigenvalues, sorts this list to find the $c$ smallest eigenvalues and then thresholds the corresponding eigenfunctions. The eigenvalue list creation step is consistent with the perspective of SSCT, however it is somewhat heuristic~\cite{ITQ}. AGH and DGH compute D2S to form up a highly sparse affinity matrix to minimize the modified object function of SH. GHS-DD avoids the computation and storage of pairwise distances of all data points by minimizing the quantization loss. Furthermore, our method generates a balanced code matrix but they cannot.\medskip\\
\subsection{Spherical Hashing (SpH)}
The final step of SpH~\cite{SpH} is the same as our method, so SpH also generates a balanced code matrix. However, SpH searches the locations of special points in the entire space, which makes it difficult to find a good solution. The authors claimed that the distances between these points should be neither too large nor too small, and hence an empirical point-finding procedure was devised that has less theoretical support. With more concrete theoretical analysis, our proposed method appears to outperform SpH.

{%\setlength{\parskip}{1\baselineskip}
%\begin{table*}[b]
%\caption{MAP on SUN397, GIST1M and SIFT10M. }$c$ denotes the number of hash bits used in hashing methods.}
%\label{tb:MAP}
%\begin{center}
%\begin{tabular}{|c|c|c|c|c|c|c|c|c|c|}
%\hline
%&\multicolumn{3}{c}{\bf SUN397} \vline  &\multicolumn{3}{c}{\bf GIST1M} \vline & \multicolumn{3}{c}{\bf SIFT10M} \vline
%\\ \hline
%$c$	&8	&32	&128	&8	&32	&128	&8	&32	&128
%\\ \hline
%{\bf GHS-DI}	&{\bf 0.1336}	&{\bf 0.2579}	&{\bf 0.3860}	&{\bf 0.1245}	&{\bf 0.2191}	&{\bf 0.2985}	 &{\bf 0.1738}	&{\bf 0.3837}	 &{\bf 0.5797}
%\\ \hline
%{\bf GHS-DD}	&{\bf 0.1533}	&{\bf 0.2898}	&{\bf 0.4096}	&{\bf 0.1358}	&{\bf 0.2438}	&{\bf 0.3277}	 &{\bf 0.1864}	&{\bf 0.4098}	 &{\bf 0.5889}
%\\ \hline
%ITQ	&0.1508	&0.2886	&0.3750	&0.1260	&0.2269	&0.2775	 &0.1666	&0.3906	&0.5782
%\\ \hline
%IsoH	&0.1420	&0.2278	&0.2882	&0.1121	&0.2288	 &0.2854	&0.1764	&0.3766	&0.5695
%\\ \hline
%HH	&0.1478	&0.2687	&0.3739	&0.1207	&0.2247	&0.1781	 &0.1704	&0.2810	&0.3157
%\\ \hline
%IMH	&0.1296	&0.2689	&0.3990	&0.1248	&0.1965	&0.2638	 &0.1833	&0.2884	&0.3634
%\\ \hline
%okmeans	&0.1469	&0.2716	&0.3658	&0.1239	&0.2201	 &0.2809	&0.1814	&0.3605	&0.4964
%\\ \hline
%SpH	&0.0377	&0.0363	&0.2578	&0.0369	&0.0356	&0.1919	 &0.0440	&0.0381	&0.1947
%\\ \hline
%\end{tabular}
%\end{center}
%\end{table*}
%}
{\setlength{\parskip}{1\baselineskip}
\begin{figure*}[htb]
\begin{center}
%\framebox[4.0in]{$\;$}
%\fbox{\rule[-.5cm]{0cm}{4cm} \rule[-.5cm]{4cm}{0cm}}
\includegraphics[width=0.9\linewidth]{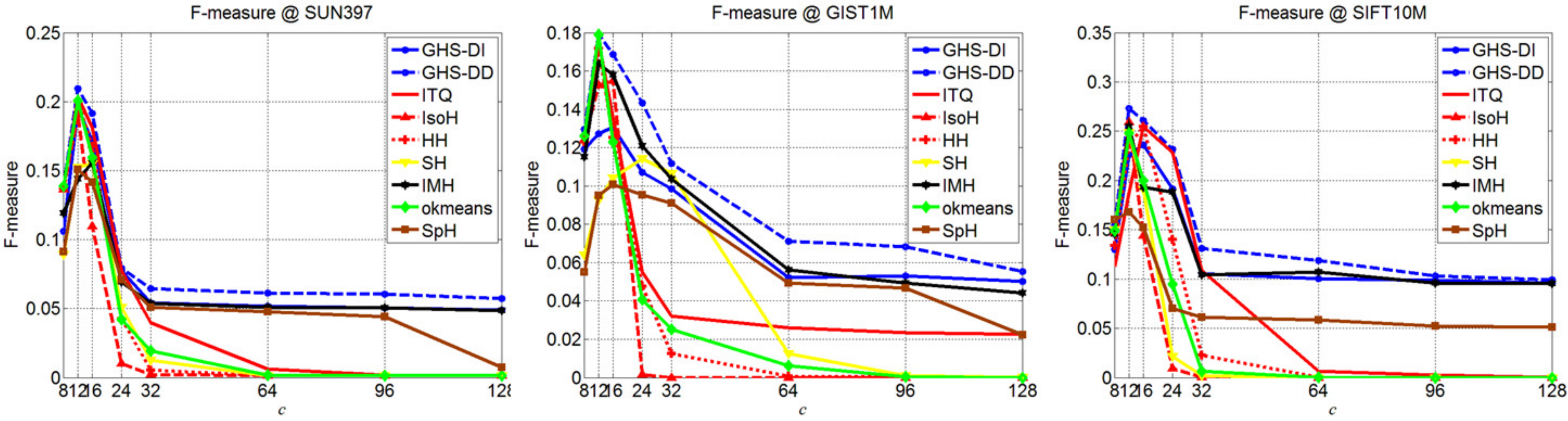}
\end{center}
\caption{Mean F-measure of hash lookup with Hamming radius 2 for different methods on SUN397, GIST1M and SIFT10M. }
\label{fig:F}
\end{figure*}
}
\begin{table*}[htb]
\caption{MAP on SUN397. $c$ denotes the number of hash bits used in hashing methods.}
\label{tb:MAP1}
\begin{center}
\begin{tabular}{|c|c|c|c|c|c|c|c|c|}
\hline
& \multicolumn{8}{c}{\bf SUN397} \vline
\\ \hline
$c$	& 8	&12	&16	&24	&32	&64	&96	&128\\ \hline
\bf GCS-DI	&\bf 0.1336	&\bf 0.1744	&\bf 0.2194	&\bf 0.2290	&\bf 0.2579	&\bf 0.3167	 &\bf 0.3588	&\bf 0.3860\\ \hline
\bf GCS-DD	&\bf 0.1533	&\bf 0.1945	&\bf 0.2447	&\bf 0.2746	&\bf 0.2998	&\bf 0.3492	 &\bf 0.3880	&\bf 0.4096\\ \hline
ITQ	&0.1508	&0.1859	&0.2301	&0.2619	&0.2886	&0.3317	&0.3592	&0.3750\\ \hline
IsoH	&0.1420	&0.1677	&0.1881	&0.1950	&0.2278	&0.2578	&0.2873	&0.2882\\ \hline
HH	&0.1478	&0.1866	&0.2213	&0.2554	&0.2687	&0.3253	&0.3543	&0.3739\\ \hline
SH	&0.1219	&0.1369	&0.1475	&0.1705	&0.1758	&0.1897	&0.2180	&0.2206\\ \hline
IMH	&0.1296	&0.1357	&0.1533	&0.2453	&0.2689	&0.2896	&0.3077	&0.3990\\ \hline
okmeans	&0.1469	&0.1852	&0.2136	&0.2524	&0.2716	&0.3248	&0.3507	&0.3658\\ \hline
SpH	&0.0377	&0.0359	&0.0364	&0.0365	&0.0363	&0.0599	&0.0942	&0.2578\\ \hline
\end{tabular}
\end{center}
\end{table*}
\begin{table*}[htb]
\caption{MAP on GIST1M. $c$ denotes the number of hash bits used in hashing methods.}
\label{tb:MAP2}
\begin{center}
\begin{tabular}{|c|c|c|c|c|c|c|c|c|}
\hline
& \multicolumn{8}{c}{\bf GIST1M} \vline
\\ \hline
$c$	&8	&12	&16	&24	&32	&64	&96	&128\\ \hline
\bf GCS-DI	&\bf 0.1245	&\bf 0.1552	&\bf 0.1802	&\bf 0.2052	&\bf 0.2191	&\bf 0.2596	 &\bf 0.2790	&\bf 0.2885\\ \hline
\bf GCS-DD	&\bf 0.1358	&\bf 0.1682	&\bf 0.1952	&\bf 0.2211	&\bf 0.2438	&\bf 0.2694	 &\bf 0.2854	&\bf 0.2967\\ \hline
ITQ	&0.1260	&0.1593	&0.1851	&0.2098	&0.2269	&0.2577	&0.2703	&0.2775\\ \hline
IsoH	&0.1121	&0.1310	&0.1844	&0.1939	&0.2288	&0.2579	&0.2712	&0.2854\\ \hline
HH	&0.1207	&0.1603	&0.1780	&0.2019	&0.2247	&0.2597	&0.2745	&0.2880\\ \hline
SH	&0.0871	&0.0986	&0.1033	&0.1208	&0.1339	&0.1682	&0.1781	&0.1781\\ \hline
IMH	&0.1248	&0.1449	&0.1748	&0.1849	&0.1965	&0.2161	&0.2385	&0.2638\\ \hline
okmeans	&0.1239	&0.1610	&0.1778	&0.2070	&0.2201	&0.2565	&0.2741	&0.2809\\ \hline
SpH	&0.0369	&0.0349	&0.0348	&0.0359	&0.0356	&0.0637	&0.0788	&0.1919\\ \hline
\end{tabular}
\end{center}
\end{table*}

\begin{table*}[htb]
\caption{MAP on SIFT10M. $c$ denotes the number of hash bits used in hashing methods.}
\label{tb:MAP3}
\begin{center}
\begin{tabular}{|c|c|c|c|c|c|c|c|c|}
\hline
& \multicolumn{8}{c}{\bf SIFT10M} \vline
\\ \hline
$c$	&8	&12	&16	&24	&32	&64	&96	&128\\ \hline
\bf GCS-DI	&\bf 0.1738	&\bf 0.2193	&\bf 0.2674	&\bf 0.3342	&\bf 0.3837	&\bf 0.5156	 &\bf 0.5569	&\bf 0.5797\\ \hline
\bf GCS-DD	&\bf 0.1864	&\bf 0.2339	&\bf 0.2769	&\bf 0.3535	&\bf 0.4098	&\bf 0.5277	 &\bf 0.5692	&\bf 0.5889\\ \hline
ITQ	&0.1666	&0.2195	&0.2655	&0.3452	&0.3906	&0.5025	&0.5522	&0.5782\\ \hline
IsoH	&0.1764	&0.2224	&0.2469	&0.3326	&0.3766	&0.4653	&0.5524	&0.5695\\ \hline
HH	&0.1701	&0.2258	&0.2516	&0.3143	&0.3524	&0.4494	&0.5163	&0.5554\\ \hline
SH	&0.1704	&0.2170	&0.2382	&0.2708	&0.2810	&0.3148	&0.3039	&0.3157\\ \hline
IMH	&0.1833	&0.1888	&0.2007	&0.2254	&0.2884	&0.3052	&0.3358	&0.3634\\ \hline
okmeans	&0.1814	&0.2260	&0.2699	&0.3233	&0.3605	&0.4401	&0.4538	&0.4964\\ \hline
SpH	&0.0440	&0.0487	&0.0400	&0.0475	&0.0381	&0.0615	&0.1721	&0.1947\\ \hline
\end{tabular}
\end{center}
\end{table*}
{%\setlength{\parskip}{1\baselineskip}
\begin{table*}[htb]
\caption{Training and testing time in seconds}
\label{tb:time}
\begin{center}
\begin{tabular}{|c|c|c|c|c|c|c|c|}
\hline
&\multicolumn{2}{c}{\bf SUN397}\vline&\multicolumn{2}{c}{\bf GIST1M}\vline&\multicolumn{2}{c}{\bf SIFT10M}\vline
\\ \hline
&Train&Test&Train&Test&Train&Test
\\ \hline
{\bf GHS-DI}&${\bf 9.9}$&${\bf 2.7\times 10^{-4}}$&${\bf 130.4}$&${\bf 3.5\times 10^{-4}}$&${\bf 166.1}$&${\bf 1.4\times 10^{-4}}$
\\ \hline
{\bf GHS-DD} &${\bf 24.3}$ &${\bf 3.2\times 10^{-4}}$ &${\bf 212.3}$ &${\bf 3.5\times 10^{-4}}$ &${\bf 1005.1}$ &${\bf 1.4\times 10^{-4}}$
\\ \hline
ITQ & $14.8$ &$3.1\times 10^{-5}$ &$142.7$ &$4.5\times 10^{-5}$ &$322.0$ &$1.3\times 10^{-5}$
\\ \hline
IsoH & $9.6$ &$3.2\times 10^{-5}$ &$136.5$ &$6.1\times 10^{-5}$ &$185.6$ &$2.0\times 10^{-5}$
\\ \hline
HH & $26.8$ &$2.1\times 10^{-5}$& $214.9$ &$3.9\times 10^{-5}$ &$1307.1$ &$1.3\times 10^{-5}$
\\ \hline
SH & $9.7$ &$6.5\times 10^{-4}$ &$119.3$ &$9.2\times 10^{-4}$& $202.5$ &$6.2\times 10^{-4}$
\\ \hline
IMH & $97.4$ &$2.3\times 10^{-4}$ &$1024.4$ &$2.8\times 10^{-4}$ &$702.2$ &$2.8\times 10^{-4}$
\\ \hline
okmeans& $14.0$ &$2.3\times 10^{-5}$ &$144.5$ &$5.5\times 10^{-5}$ &$301.2$& $1.2\times 10^{-5}$
\\ \hline
SpH & $28.2$ &$3.3\times 10^{-4}$& $225.8$ &$4.4\times 10^{-4}$ &$190.7$ &$2.7\times 10^{-4}$
\\ \hline
\end{tabular}
\end{center}
\end{table*}
}
\begin{figure*}[t]
\begin{center}
%\framebox[4.0in]{$\;$}
%\fbox{\rule[-.5cm]{0cm}{4cm} \rule[-.5cm]{4cm}{0cm}}
\includegraphics[width=0.8\linewidth]{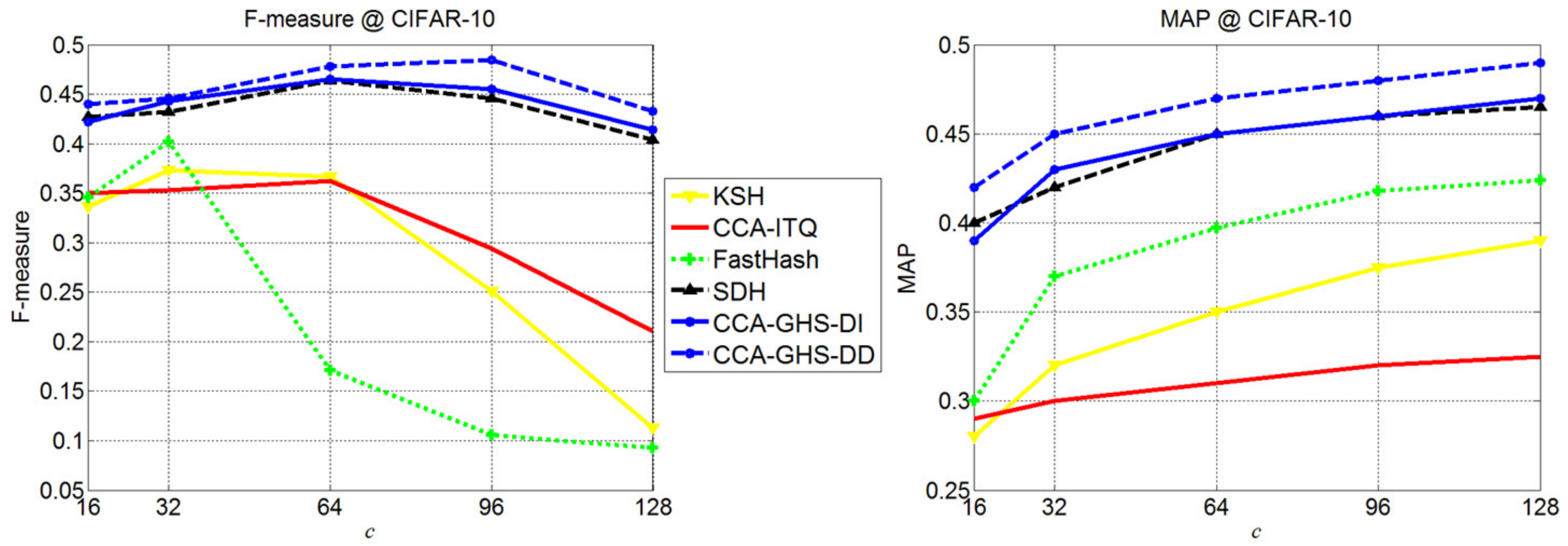}
\end{center}
\caption{Mean F-measure of hash lookup with Hamming radius 2 and MAP for different methods on CIFAR-10.}
\label{fig:CIFARSupervised}
\end{figure*}
\begin{figure*}[htb]
\begin{center}
\includegraphics[width = 0.8\linewidth,trim=10 100 50 200,clip]{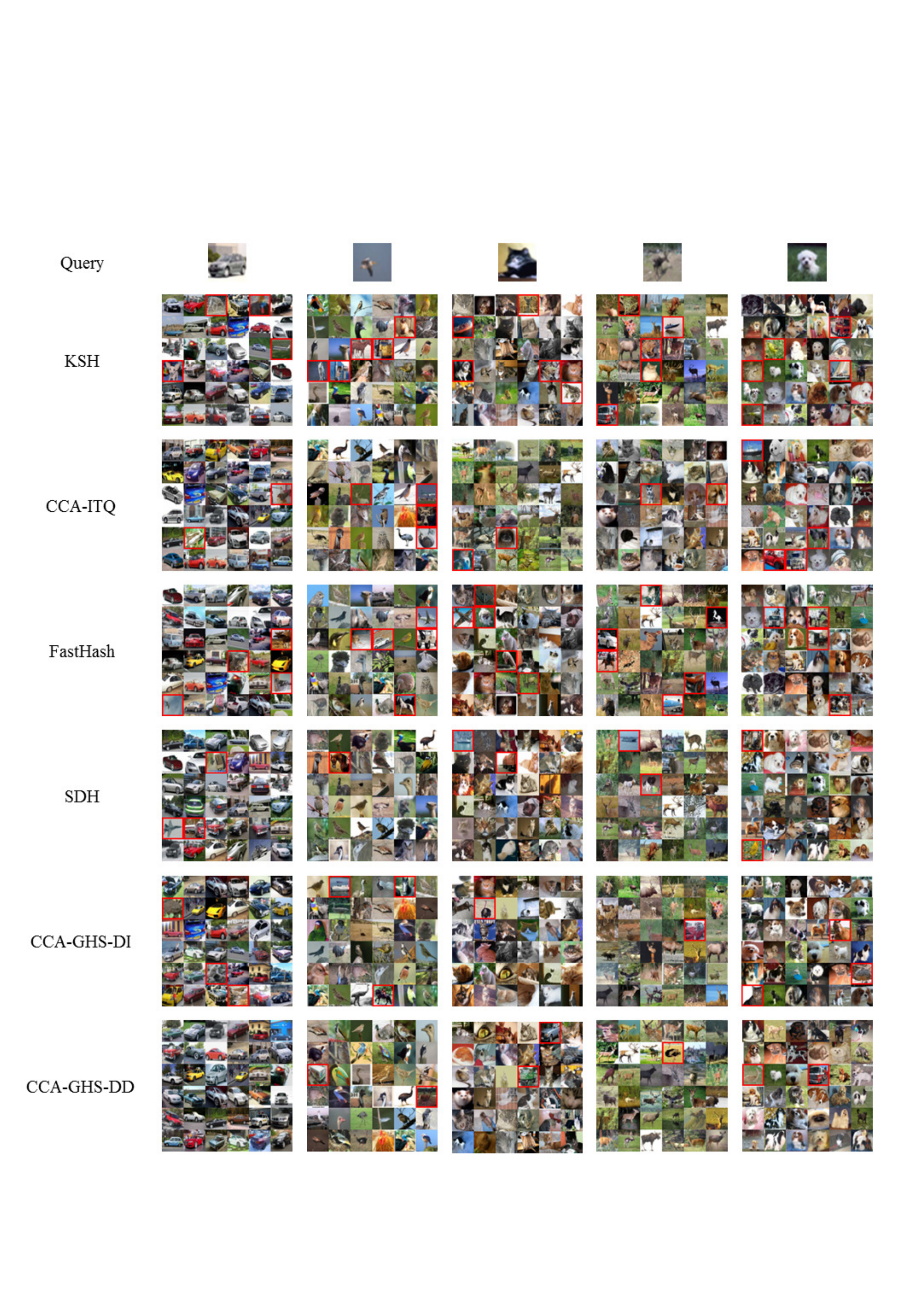}
\caption{The query images and the query results returned by compared methods with 32 hash bits.}
\label{fig:synthesis}
\end{center}
\end{figure*}
\section{Experiments}
\label{sec:exp}
Our experiments were conducted on three datasets of three different scales: SUN397~\cite{SUN397}, GIST1M~\cite{GIST1M} and SIFT10M. SUN397 contains about 108K images and we represent each image by a 512-dimensional GIST descriptor~\cite{GIST}. GIST1M consists of 1 million 960-dimensional GIST descriptors. SIFT10M is a 10 million subset of SIFT1B~\cite{GIST1M} dataset which comprises of 1 billion 128-dimensional SIFT descriptors~\cite{SIFT}. The 10 million data points are randomly chosen. 1K images are randomly selected from the whole SUN397 to form a separate test dataset. For GIST1M, there is a 1K test dataset available. For SIFT10M, we randomly selected 1K data points from its 10K test dataset. Groundtruth neighbors for a given query are defined as the samples in the top of 2\% Euclidean distance.
\subsection{Protocols and Baselines}
We evaluate our methods by comparing to seven hashing methods which includes: Iterative Quantization (ITQ)~\cite{ITQ}, Isotropic Hashing (IsoH)~\cite{IsoH}, Harmonious Hashing (HH)~\cite{HH}, Spectral Hashing (SH)~\cite{SH}, Inductive Manifold Hashing (IMH)~\cite{IMH}, Orthogonal K-means (ok-means)~\cite{okmeans} and Spherical Hashing (SpH)~\cite{SpH}. Our data-dependent and data-independent are denoted as GHS-DD and GHS-DI, respectively. We use publicly available codes of comparing methods and follow the suggesting parameter settings by corresponding publications. All data are zero-centered and in our methods, their PCA projections are normalized by the largest Euclidean norm of all projected data in our methods. Two kinds of experiments - \emph{Hamming ranking} and \emph{hash lookup} were conducted. The performance of \emph{Hamming ranking} is measured by MAP and F1 score which is denoted as F-measure is used for evaluating the performance of \emph{hash lookup}, where F1 score is defined as $2(precision\cdot recall)/(precision+recall)$. Ground truths are defined by Euclidean neighbors.

\subsection{Quantitative Evaluation}
The mean average precision (MAP) values are given in Table~\ref{tb:MAP1}-\ref{tb:MAP3}. It can be seen that GHS-DD outperforms all compared methods. The performance of GHS-DI is poorer than ITQ, HH and SH except of 128-bit experiments. Benefitting from the reasonability on information theory and balanced code matrix, GHS-DD exceeds ITQ, IsoH and HH. Due to the limitation on computation, SpH works on a small subset of the whole dataset and its empirical satellite distribution algorithm is demonstrated to be less efficient than ours. The F-measure is illustrated in Fig.~\ref{fig:F}. Again, GHS-DD exceeds others. It is worth noticing that GHS-DI generated the second best MAP and F-measure in experiments on longer bits ($c>96$), because GHS-DI considers orthogonality of the code matrix. The way that GHS-DD satisfies the condition of uniqueness and existence of GPS solution, \emph{i.e.}, Eq.~\eqref{eq:cond} and its data-dependent property makes it work better than GHS-DI.
\subsection{Computational Efficiency}
Training and testing time on 32-bit are given in Table~\ref{tb:time}. All experiments were done on MATLAB R2013b installed on a PC with 2.85 GHz CPU and 128 GB RAM. The major computation cost of GHS-DI is the calculation of D2S at the final step, which is linearly related to the product of data dimension and size of dataset. Hence, it takes the least time on GIST1M and SIFT10M. Because GHS-DD computes D2S in every iteration, its computation cost is moderate. When testing a new query, GHS-DI and GHS-DD computes D2S and hence their computation costs are approximate. Although the testing procedure of SpH is similar to ours, it computes D2S in original input data space whose dimension is $D$, so its testing time is longer.

\subsection{Incorporating Label Information}
To incorporate label information, a supervised dimensionality reduction
method can be used to better capture the semantic structure of the
dataset. Among various supervised dimensionality reduction methods,
Canonical Correlation Analysis (CCA)~\cite{CCA} has proven to be efficient
for extracting a common latent space from two views~\cite{CCA_guanshui1} and robust
to noise~\cite{CCA_guanshui2}.\\

\indent Let $\mathbf{z_{i}\in}\{0,1\}^{l}$ be a label vector, where $l$
is the total number of labels. If the $i$th image is associated with
the corresponding label, $\mathbf{z_{i}=}1$ and $\mathbf{z_{i}=}0$
otherwise. $\mathbf{Z\in}\{0,1\}^{n\times l}$ is the matrix whose
rows are comprised of label vectors. The goal of CCA is to maximize
the correlation between projected data matrix\textbf{ $\mathbf{Y}$}
and label matrix $\mathbf{Z}$ by finding two projection directions
$\mathbf{w_{k}}$ and $\mathbf{u_{k}}$. The correlation is defined
as:

\begin{equation}
\begin{aligned}C\left(\mathbf{w_{k},u_{k}}\right)=\frac{\mathbf{w_{k}^{\top}X^{\top}Yu_{k}}}{\sqrt{\mathbf{w_{k}^{\top}X^{\top}Xw_{k}u_{k}^{\top}Y^{\top}Yu_{k}}}}\\
s.t.\, w_{k}^{\top}X^{\top}Xw_{k}=1,\, u_{k}^{\top}Y^{\top}Yu_{k}=1.
\end{aligned}
\end{equation}
 $\mathbf{w_{k}}$ can be got by solving the following generalized
eigenvalue problem:

\begin{equation}
\mathbf{X^{\top}Y}\left(\mathbf{Y^{\top}Y}+\rho\mathbf{I}\right)^{-1}\mathbf{Y^{\top}}\mathbf{Xw_{\mathbf{k}}}=\lambda_{k}^{2}\left(\mathbf{X^{\top}X}+\rho\mathbf{I}\right)\mathbf{w}_{\mathbf{k}}\label{eq:wk},
\end{equation}
where $\rho$ is a small regularization constant and is set to be
0.0001 here. Just as in the case of PCA, the leading generalized eigenvectors
$\mathbf{w}_{\mathbf{k}}$ scaled their corresponding eigenvalues $\lambda_k$ form up the rows of
projection matrix $\mathbf{\widehat{W}}\in\mathbb{R}^{D\times d}$
and we obtain the embeded data matrix $\mathbf{Y=X\widehat{W}}$.
Finally, both of our data-independent and data-dependent methods can
be used to generate hashing codes.\\

\indent CIFAR-10 dataset is used in this experiment. The 60K images in CIFAR-10
are labelled as 10 classes with 6,000 samples for each class. Again,
each image is represented by a 1024 dimensional GIST feature. 1,000
samples are randomly chosen as queries and the remaining samples are
used for training. Our proposed supervised hashing methods are denoted as CCA-GHS-DI and CCA-GHS-DD, respectively. The baseline methods are Supervised Discrete Hashing (SDH)~\cite{SDH}, KSH~\cite{CV2}, FastHash~\cite{FastHash} and CCA-ITQ~\cite{ITQ}.\\

\indent The mean F-measure of hash lookup Hamming distance 2 and MAP scores of the compared methods are given in Fig.~\ref{fig:CIFARSupervised}. CCA-GHS-DD achieves the best F-measures and MAPs for all code lengths, while CCA-GHS-DI is only a little inferior to SDH for 16-bit code length. In the hash lookup experiments, we found that setting Hamming distance as 2 is favorable for both of our proposed methods, because two groups of satellites were used for experiments of $c>16$. In Fig.~\ref{fig:synthesis}, 5 queries with their corresponding results retrieved by compared methods using 16-bit hashing code are illustrated to qualitatively evaluate the performance. It can be seen that both CCA-GHS-DI and CCA-GHS-DD outperform the compared methods.

\begin{figure}[t]
\begin{center}
\includegraphics[width = 0.9\linewidth]{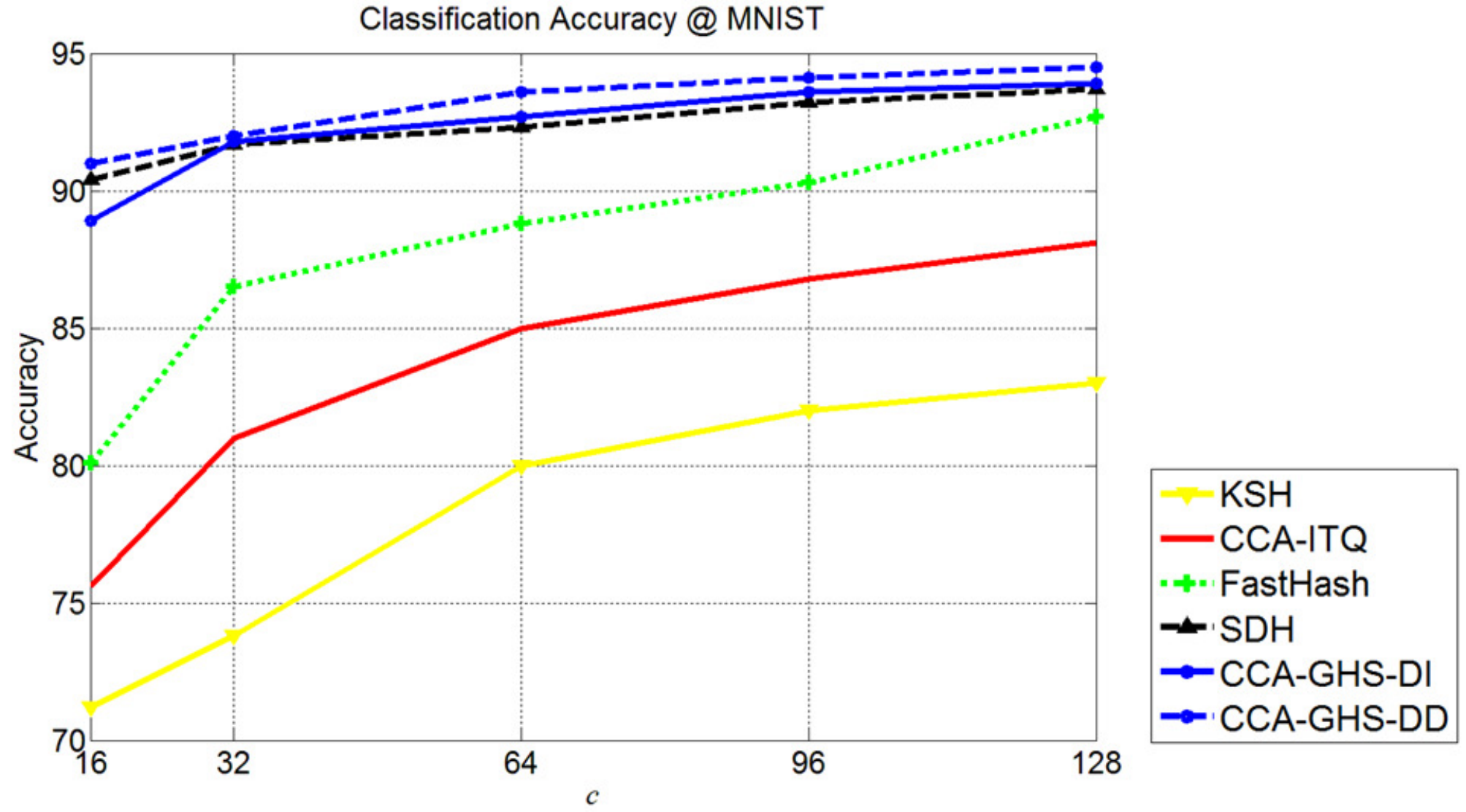}
\caption{Classification accuracy (\%) on MNIST}
\label{fig:classification}
\end{center}
\end{figure}
\subsection{Classification with hashing codes}
In this subsection, the MNIST dateset is used for evaluate the performance of the learned hashing codes by compared methods. The MNIST dataset consists of 70, 000 images, each of which is 784-dimensional. These images are handwritten digits from `0' to `9'. BRE, CCA-ITA, KSH, FastHash and SDH are used as baselines.\\
\indent Linear Support Vector Machine (SVM) is applied on the hashing codes. The LIBLINEAR~\cite{libsvm} solver is used to train the SVM. The classification results are given in Fig.~\ref{fig:classification}. From Fig.~\ref{fig:classification}, it can be seen that both CCA-GHS-DD gets the highest classification accuracy over all hash bit length, while CCA-GHS-DI is the second best when $c>32$ but trail SDH in experiments on 32-bit hash codes.
\section{Conclusion}

We have proposed a novel hashing method based on and Shannon's Source Coding Theorem witch requires that the hashing codes should be longer than the embedding for original training data. To circumvent computation of pairwise distances between each pair of data points, we minimize the new formulation of quantization loss which is based on Global Positioning System (GPS). Data-dependent and data-independent methods are proposed to distribute the satellites. According to the experimental results on three scales of datasets, the data-dependent method (GHS-DD) was superior to other methods, and the data-independent method (GHS-DI) produced promising results in less training time. However, GHS-DD took a moderate length of time to train, and the demand on RAM was limited by the computation of the covariance matrix in PCA. By incorporating Canonical Correlation Analysis (CCA), the proposed methods can be used for supervised hashing. The performance of CCA-GHS-DI and CCA-GHS-DD are superior. Finally, the retained hashing codes are used for classification problem to further demonstrate the outstanding performance of the proposed methods. Future work will focus on improving the computational efficiency and investigating methods to train the model using a few samples from the whole dataset to handle larger datasets such as SIFT1B and Tiny 80M.

\bibliographystyle{IEEEtran}
\bibliography{Citation}
\end{document}